\documentclass[10pt,conference]{IEEEtran}
\IEEEoverridecommandlockouts
\usepackage{cite}
\usepackage{times}
\usepackage{soul}
\usepackage{url}
\usepackage[hidelinks]{hyperref}
\usepackage[utf8]{inputenc}
\usepackage[small]{caption}
\usepackage{graphicx}
\usepackage{amsmath}
\usepackage{amsthm}
\usepackage{booktabs}
\usepackage{listings}

\usepackage[ruled,linesnumbered]{algorithm2e}
\usepackage[most]{tcolorbox}

\usepackage{xcolor}      
\usepackage{hyperref}    
\usepackage{rotating}  
\usepackage{colortbl}
\usepackage{mathtools}
\usepackage{pgfplots}
\pgfplotsset{compat=1.18}

\usepackage{amssymb}  
\usepackage{threeparttable} 
\usepackage{booktabs}       
\newtheorem{definition}{Definition}

\usepackage{titletoc}
\usepackage{fontawesome5}
\newtheoremstyle{rqplain}
  {1pt}{0pt}{\bfseries}{0pt}{\bfseries}{.}{0.5em}
  {\thmname{#1}\thmnumber{#2}\thmnote{ (#3)}}
\theoremstyle{rqplain}
\newtheorem{researchquestion}{RQ}

\newcommand{\RQicon}{\textcolor{black}{\faQuestionCircle}\,}
\newcommand{\RRQicon}{\textcolor{black}{\faLightbulb}\,}

\makeatletter
\let\old@researchquestion\researchquestion
\def\researchquestion{%
  \def\@begintheorem##1##2[##3]{%
    \item[\hskip\labelsep\normalfont\bfseries\RQicon ##1\ ##2\if\relax\detokenize{##3}\relax\else\ (##3)\fi.]%
  }%
  \old@researchquestion
}
\makeatother

\theoremstyle{plain}

\usepackage{multirow}
\usepackage{makecell}
\usepackage{bm}
\usepackage{cleveref}
\crefname{figure}{Fig.}{Figs.}
\crefname{table}{Tab.}{Tabs.}
\crefname{equation}{Eq.}{Eqs.}
\crefname{section} {Sec.}{Secs.}
\crefname{chapter} {Ch.}{Chs.}
\crefname{appendix}{App.}{Apps.}

\definecolor{highlight}{HTML}{C8E6C9}
\definecolor{highlight2}{HTML}{DCEDC8}
\definecolor{highlight3}{HTML}{F9FBE7}

\newtcolorbox{boxK}{
    enhanced,
    colback=gray!5,    
    colframe=gray!40,
    boxrule=0.4pt,
    arc=1.5pt,
    left=5pt, right=5pt, top=3pt, bottom=3pt,
    before upper=\parindent0pt, 
}

\newcommand{\findingbox}[1]{%
\begin{boxK}
\footnotesize\textbf{\RRQicon Finding.}~#1
\end{boxK}
}

\newtheorem{proposition}{Proposition}

\hypersetup{
    colorlinks=false,          
    citebordercolor={0 1 0},   
    linkbordercolor={1 0 0},   
    urlbordercolor={0 1 1},    
    pdfborder={0 0 2}          
}

\begin{document}

\title{CIT-CAD: Constraint Intent Tree-based CAD Code Generation and Verification}

\author{\IEEEauthorblockN{Anonymous Author(s)}
}

\author{
Yali Du$^\dagger$,
Hui Sun$^\dagger$,
San-Zhuo Xi,
Ming Li\\
National Key Laboratory for Novel Software Technology,
School of Artificial Intelligence, \\Nanjing University, China \protect\\
\{duyl, sunh, lim\}@lamda.nju.edu.cn%

\IEEEcompsocthanksitem 
$^\dagger$ These authors contributed equally.
Corresponding author: Ming Li.
}


\maketitle

\begin{abstract}
Natural-language Computer-Aided Design (CAD) code generation aims to turn design intent into executable and editable parametric programs.
Large language models (LLMs) make this goal increasingly practical, but useful systems must preserve the construction process behind the rendered geometry.
Existing benchmarks and methods mostly focus on how closely the generated CAD model matches the reference geometry, often using metrics such as Intersection over Union (IoU).
Such metrics can miss errors in part decomposition, construction hierarchy, Boolean operations, sketch structure, and geometric relations.
This gap calls for a representation that makes design intent explicit and lets a system check generated code against that intent.
We propose CIT-CAD, a framework that infers a Constraint Intent Tree (CIT) from the input description to represent the intended entities, hierarchy, operations, and relations.
The tree has two roles: it guides CAD code generation and defines expected constraints for verification.
The framework extracts actual constraints from the generated program, compares them with the expected constraints, and uses mismatches to localize and repair design violations.
Experiments show that the framework improves CAD generation performance, with larger gains on more complex multi-entity designs.
By turning design intent into an explicit and checkable object, this work is the first attempt to move text-to-CAD generation beyond rendered-geometry matching toward construction-aware synthesis, verification, and repair. 
The complete artifacts are released at \url{https://anonymous.4open.science/r/CIT-CAD-2FC2}.
\end{abstract}


\section{Introduction}
\label{sec:Introduction}

Natural-language Computer-Aided Design (CAD) code generation studies how to produce executable CAD programs from design descriptions.
Recent large language models~(LLMs) have made this task increasingly feasible by generating CAD code from informal instructions~\cite{khan2024text2cad,doris2025cad,guan2025cad,he2025cad,Alrashedycad}.
Unlike a final geometric output, a CAD program is an editable construction procedure that records the modeling steps and parameters used to construct the object.
These steps include 2D sketches, extrusions that turn sketches into 3D features, Boolean operations that add or cut parts, and geometric relations among parts. Parametric programming using CadQuery\footnote{\url{https://github.com/cadquery/cadquery}} makes CAD code generation more amenable to LLMs by representing CAD models as executable construction procedures with explicit entities, parameters, and Boolean operations.

Recent work has explored CAD program generation from design descriptions and visual inputs using LLMs or vision-language models~\cite{guan2025cad,he2025cad,Alrashedycad,zhou2026cad,alam2024gencad}.
These studies demonstrate the potential of automated CAD code generation for turning informal specifications into executable programs.
Their evaluations commonly report execution success and final-geometry similarity, typically using Intersection over Union~(IoU) over the occupied regions of the 3D models produced by generated and reference CAD code~\cite{khan2024text2cad,guan2025cad,he2025cad,niu2025cad}.
These geometric measures are necessary because a generated CAD program must produce the requested geometry.

Final-geometry similarity is therefore necessary, but it is a coarse signal for CAD code generation.
In realistic generation, a model may partially match the reference while violating construction constraints that determine how the model should be edited.
An IoU score cannot localize whether an error comes from wrong decomposition, a missing feature, an incorrect Boolean role, or a violated relation between parts.
Thus, natural-language CAD code generation should evaluate not only final geometry, but also the intended construction constraints behind it.

\begin{figure}[t]
    \centering
    \includegraphics[width=0.9\linewidth]{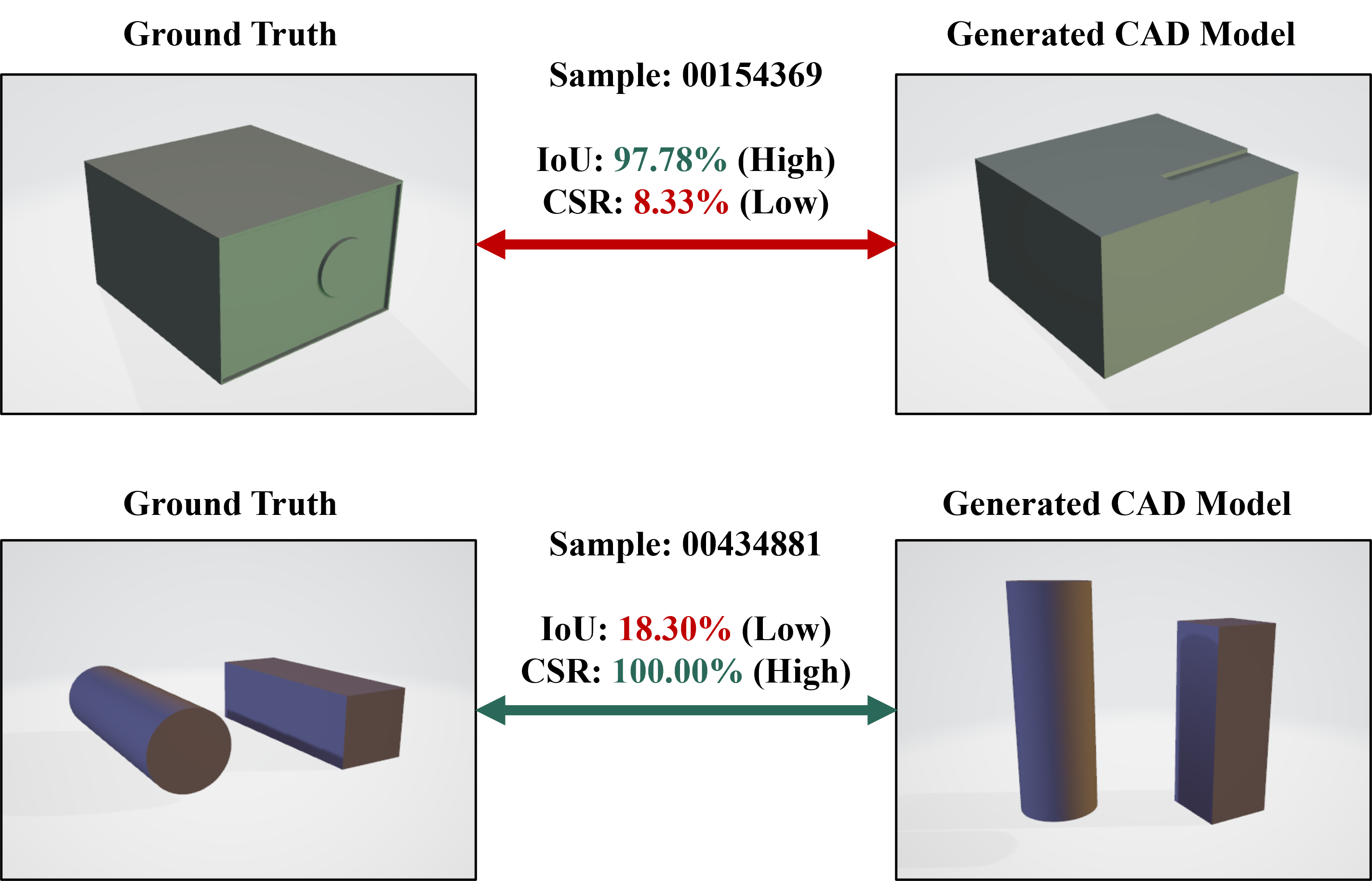}
    \caption{Motivating example for construction-aware CAD evaluation. Geometric-level similarity can hide program-level errors such as missing entities, incorrect Boolean roles, or violated relations, while CIT-CAD exposes these errors as explicit constraint violations.}
    \label{fig:motivation-example}
\end{figure}

This motivates a construction-aware view of CAD code generation.
Instead of treating generation as text-to-code followed mainly by final-geometry scoring, a system should make the intended construction process explicit.
Natural-language descriptions rarely enumerate every requirement, but they often imply expected entities, operations, and relations.
Once made explicit, these requirements can provide a common target for generation, verification, and repair.

Under this view, the intended construction constraints are represented by a Constraint Intent Tree~(CIT).
A CIT records the expected CAD entities, their construction hierarchy, the operations that combine them, and the relations that should hold among them.
It is intentionally lighter than a complete CAD program: rather than inferring every missing dimension or coordinate, it records the construction constraints that are important for program correctness.
This makes the intended design explicit before code generation and comparable after code generation.

Building on the CIT, we propose CIT-CAD, a framework that uses the inferred construction intent as a common reference for natural-language CAD code generation, verification, and repair.
Given a design description, CIT-CAD first infers a CIT and then generates executable CAD code conditioned on both the original description and this CIT.
For verification, CIT-CAD derives expected constraints from the CIT and extracts actual constraints from the generated CAD program through deterministic program and geometry analysis.
It then compares the two to localize violated constraints, which are converted into repair feedback for the LLM.
A revised program is accepted only when it preserves previously satisfied constraints and reduces the remaining violations.
In this way, CIT-CAD keeps generation, verification, and repair aligned around construction constraints rather than only final geometry.

We evaluate CIT-CAD on multi-entity samples derived from Text2CAD~\cite{khan2024text2cad}, comparing it with direct LLM-based CAD code generation.
The evaluation combines valid-program rate, final-geometry similarity, and constraint satisfaction.
CIT-CAD improves overall performance, especially on structurally complex designs, suggesting that explicit construction intent improves robustness beyond geometry matching alone.

This paper makes the following contributions:
\begin{itemize}
    \item \textbf{Constraint Intent Tree for CAD code generation.}
    We introduce a tree structure that captures expected CAD entities, construction hierarchy, operations, and relations implied by natural-language descriptions.
    \item \textbf{CIT-based CAD generation and verification.}
    We propose CIT-CAD, which infers a CIT, generates CAD code conditioned on it, and verifies generated programs against CIT-derived constraints.
    \item \textbf{Localized constraint repair.}
    We convert entity-level and relation-level violations into repair feedback and accept a repair only if it preserves satisfied constraints and reduces remaining violations.
    \item \textbf{Empirical evidence beyond geometry similarity.}
    We evaluate CIT-CAD on multi-entity Text2CAD samples and show consistent improvements in executability and constraint satisfaction across multiple LLM backbones.
\end{itemize}

\section{Related Work}
\label{sec:RelatedWork}

CIT-CAD constructs CAD models via three core steps, generating CAD scripts from natural language or visual prompts, clarifying generation objectives via intermediate representations, and validating synthesized CAD programs.

\subsection{Parametric CAD Code Generation and Benchmarks}

Parametric CAD programs encode both geometric shapes and editable construction history, forming the foundation of learning-based CAD synthesis. Existing datasets and paradigms support structured program generation, including ABC~\cite{koch2019abc}, Fusion 360 Gallery~\cite{willis2021fusion}, DeepCAD~\cite{DeepCAD}, and SketchGraphs~\cite{sketchgraphs2020}, as well as general programmatic shape generation methods such as CSGNet~\cite{csgnet2017}, ShapeAssembly~\cite{jones2020shapeassembly}, and CAD sequence modeling techniques~\cite{tian2019shapeprograms,ganin2021cadlanguage,karadeniz2024davinci}.

Compared with vision- or point-cloud-based CAD reconstruction, text-to-CAD generation is more practically demanding and technically challenging. Natural language serves as the most intuitive and universal interaction interface for novice and expert designers, making text-conditioned CAD generation the most accessible and widely required design paradigm in real-world workflows. Recent text-to-CAD benchmarks and optimization methods~\cite{khan2024text2cad,guan2025cad,he2025cad,govindarajan2025cadmium,liao2025automated} have validated the feasibility of LLM-based CAD programming.
In contrast, multimodal CAD methods relying on drawing~\cite{ma2024draw}, point clouds~\cite{rukhovich2025cad}, or 3D inputs~\cite{Alrashedycad,doris2025cad,xu2024cad,wang2025cad,yuan2024openecad,zhao2024chatcad+} require specialized input data and are limited to specific reconstruction scenarios, resulting in less general applicability.
Nevertheless, current text-to-CAD evaluations remain limited. Most approaches rely on executability checks and geometric metrics such as IoU and Chamfer distance~\cite{zhou2026cad,niu2025cad}, which only measure final shape similarity. Generated CAD programs can produce visually identical geometries while deviating significantly in sketch structures, Boolean roles, and hierarchical modeling logic. 

\subsection{Intermediate Intent Representations for Code Generation}

Code generation research often introduces intermediate representations to make user intent easier for LLMs to follow.
Code representation models connect natural language and programs~\cite{codebert}, while outline-based, coarse-to-fine, and structured reasoning methods use program sketches or intermediate steps to make generation more controllable~\cite{zheng2023outline,li2025structured,yang2024chain}.
Prompting and example-selection work provides a complementary form of guidance, showing that code structure can matter when constructing LLM inputs~\cite{li2025large,yali_tse25_cast}.
Structure-aware code representations have also been used for bug localization, code-change modeling, and commit-message generation, where semantic-flow, control-flow, and context-aware representations help preserve program behavior beyond token sequences~\cite{GraphCodeBERT,du2023pre,ma2023capturing,du2025capturing}.
Recent software-engineering surveys further show that LLM-based development faces persistent challenges in controllability, evaluation, efficiency, and trustworthiness~\cite{fan2023large,gao2025challenges,shi2025greenllmse}.
Formal studies of code-generation behavior examine license compliance, generalization beyond familiar problems, harmful off-the-shelf use, code simplification, and code translation~\cite{xu2025licoeval,zhang2025unseen,alkaswan2025codered,wang2024natural,yang2024codetranslation}.
Recent ICSE studies further evaluate LLMs in pragmatic and class-level code-generation settings, showing that realistic program structure and context remain difficult for general-purpose code models~\cite{yu2024codereval,du2024classlevel}.
Several studies further show that preserving structural, semantic, and execution-state information is important when translating code across languages or long contexts~\cite{du2023beyond,du2024joint,xin2025enhancing}.
These studies suggest that explicit intermediate objects or structurally informed prompts can improve controllability before code generation.
For CAD code generation, the relevant intermediate object extends beyond a syntactic plan or a sequence of coding steps.
It should represent construction intent, such as expected entities, hierarchy, Boolean roles, sketch properties, and inter-entity relations.

\subsection{Verification and Repair of Generated Programs}

Generated code often benefits from feedback after the first attempt.
General code-generation work evaluates and improves model outputs through execution, tests, reranking, or robustness analysis~\cite{cassano2023multiple,mastropaolo2023robustness,zhang2023coder}.
Interactive and test-driven workflows further show that feedback can guide users and models toward better programs~\cite{fakhoury2024llm}.
LLM-based testing research has studied unit-test generation, mutation testing, fuzzing, test-oriented model fine-tuning, and consistency-oriented evaluation of generated code~\cite{wang2024llmtesting,yuan2024chatgptunittest,tang2024chatgptvs,tip2025llmorpheus,deng2023zeroshotfuzzers,shang2025unittestfinetune,zhang2024fuzzdriver,sun2026aces,xia2024fuzz4all}.
At test time, self-refinement methods use feedback and iterative revision to improve LLM outputs without additional training~\cite{madaan2023selfrefine}.
Program repair systems extend this idea with autonomous repair agents, fact-selection mechanisms, templates, fine-tuning, and hybrid program-analysis feedback that determine what information should be supplied to the model~\cite{guo2024chatgptrefinement,peng2024domainknowledge,bouzenia2025repairagent,parasaram2025fact,huang2025templateguided,ehsani2025hierarchical,li2024peftrepair,huang2023finetuningrepair,li2025hybridapr,huang2025comprehensive,yang2025revisiting}.
These feedback signals are useful and typically operate at the level of tests, compiler behavior, visual review, or global geometry.
They can reveal that a result is non-executable or geometrically inaccurate, while the repair target often remains global rather than tied to a specific expected entity, operation, or relation.

\begin{figure*}[t]
    \centering
    \includegraphics[width=\linewidth]{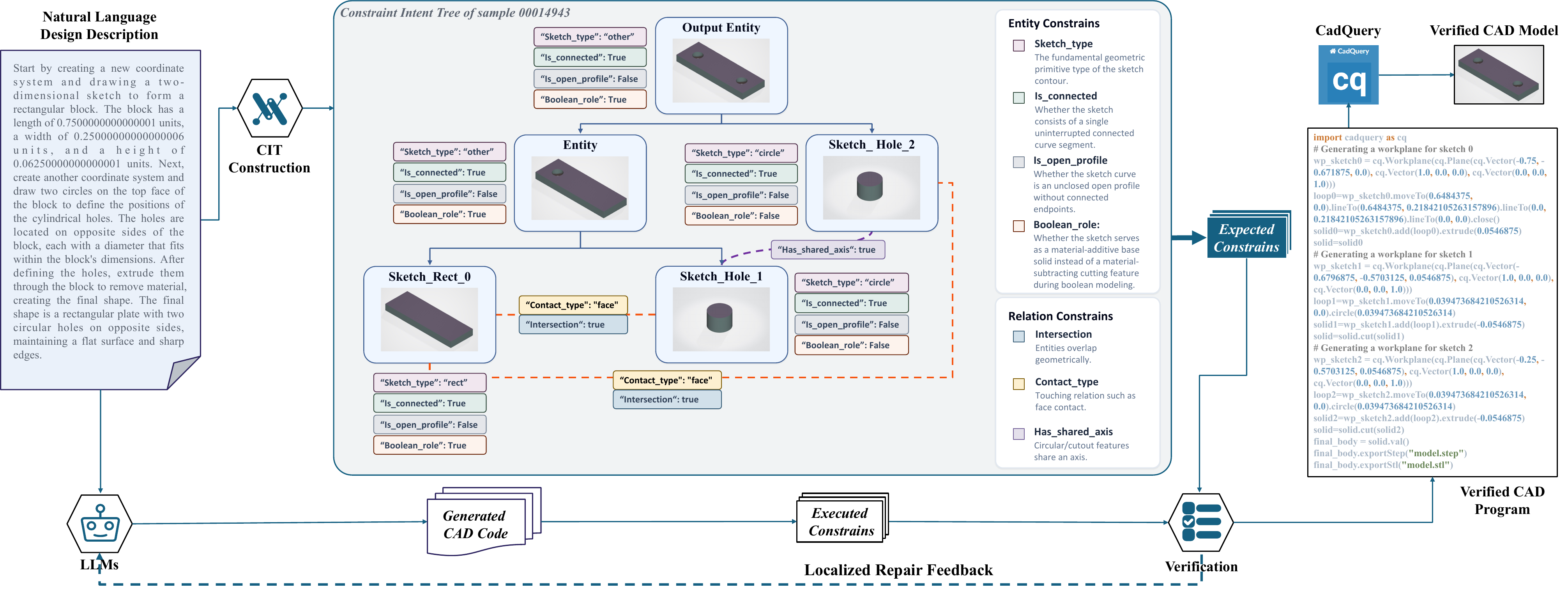}
    \caption{Overview of the framework CIT-CAD. The framework infers a Constraint Intent Tree from the natural-language description, generates CAD programs conditioned on the tree, extracts expected and executed constraints, validates constraint satisfaction, and uses localized violations to guide monotonic repair.}
    \label{fig:cit-cad-pipeline}
\end{figure*}

\begin{table*}[t]
\centering
\caption{Constraint Category in CIT-CAD.}
\label{tab:cit-constraints}
\begin{tabular}{cccc}
\toprule
\textbf{Category} & \textbf{Constraint} & \textbf{Meaning} & \textbf{Values} \\
\midrule
Entity & \lstinline|Sketch_type| & Geometric primitive category of the 2D sketch contour & \lstinline|rect|, \lstinline|circle|, etc. \\
Entity & \lstinline|Is_connected| & Indicator whether the sketch forms a single continuous connected curve & \lstinline|True|/\lstinline|False| \\
Entity & \lstinline|Is_open_profile| & Indicator whether the sketch corresponds to an unclosed polyline profile & \lstinline|True|/\lstinline|False| \\
Entity & \lstinline|Boolean_role| & Functional boolean operation type for feature construction & \lstinline|True|: additive; \lstinline|False|: subtractive \\
\midrule
Relation & \lstinline|Intersection| & Indicator whether two constructed solids have volumetric overlap & \lstinline|True|/\lstinline|False| \\
Relation & \lstinline|Contact_type| & Specific geometric contact form between adjacent solid components & \lstinline|point|, \lstinline|line|, \lstinline|face|, \lstinline|none| \\
Relation & \lstinline|Has_coplanar_faces| & Total count of coplanar face pairs across paired entities & Integer count \\
Relation & \lstinline|Has_shared_axis| & Indicator whether arc or cylinder features share a central axis & \lstinline|True|/\lstinline|False| \\
Relation & \lstinline|Tangency| & Indicator whether adjacent surfaces maintain smooth tangent connection & \lstinline|True|/\lstinline|False| \\
Relation & \lstinline|Coaxiality| & Indicator whether two structural parts share the same central axis & \lstinline|True|/\lstinline|False| \\
Relation & \lstinline|Alignment| & Spatial directional and coordinate alignment constraint between parts & \lstinline|x|, \lstinline|y|, \lstinline|z|, \lstinline|parallel|, \lstinline|perpendicular| \\
Relation & \lstinline|Symmetry| & Structural mirror or rotational symmetry relationship of components & \lstinline|mirror|, \lstinline|rotational| \\
\bottomrule
\end{tabular}
\end{table*}

\section{Method}
\label{sec:method}

This section presents CIT-CAD, a Constraint Intent Tree-based framework for natural-language CAD code generation and verification. The key idea is to introduce an intermediate representation, called a Constraint Intent Tree (CIT), between a natural-language design description and executable CAD code. A CIT captures not only the hierarchical construction structure of a CAD model, but also the implicit geometric and semantic constraints that should hold among its parts. CIT-CAD uses the CIT for two purposes: guiding LLM-based CAD program generation and deriving deterministic constraints for post-generation verification.

\subsection{Overview}
\label{subsec:method-overview}

As shown in Figure~\ref{fig:cit-cad-pipeline}, given a natural-language CAD description $D$, CIT-CAD performs four stages. First, it infers a CIT $T$ from $D$ using an LLM with a schema-constrained prompt. Second, it generates a CAD program $P$ conditioned on both $D$ and $T$. Third, it deterministically derives an expected constraint set $C_T$ from $T$, and extracts an executed constraint set $C_P$ from $P$ through static program analysis and runtime geometric reasoning. Finally, it compares $C_T$ and $C_P$ to validate the generated program, localizes violations to specific CIT nodes, relations, or subtrees, and uses the localized feedback to drive an iterative repair loop.

The framework is designed around the observation that geometry-only metrics, such as IoU, are insufficient for evaluating generated CAD programs. A program can produce plausible final geometry while using an incorrect construction hierarchy, missing a subtractive feature, violating coplanarity or shared-axis constraints, or losing editability. CIT-CAD therefore evaluates generated CAD code by checking whether it satisfies the structural and semantic constraints implied by the design intent.

\subsection{Constraint Intent Tree}
\label{subsec:cit}

\begin{definition}[Constraint Intent Tree]
\label{def:cit}
A Constraint Intent Tree is a tuple $T=(V,\mathcal{E},A,R)$, where $V$ is a set of nodes representing CAD entities or non-physical grouping nodes; $\mathcal{E}\subseteq V\times V$ is a set of parent-child edges representing hierarchical construction structure; $A$ is the set of node-level attributes and constraints; $R$ is the set of inter-node relations encoding geometric and semantic constraints among nodes.
\end{definition}

Each entity node in a CIT corresponds to an expected parametric construction entity in the generated CAD program. In our implementation, a node contains an identifier, a human-readable name, an expected code variable, node-level constraints, relation-level constraints, and optionally child nodes. Construction semantics are represented by verifiable constraints, especially the node-level \lstinline|Boolean_role| field and relation-level spatial constraints. Non-physical grouping nodes are used only to organize related entities and do not correspond to generated solids.

Node-level constraints describe properties that can be verified on individual generated entities. Let $\mathcal{F}$ denote the set of node-level constraint fields, and let $\mathcal{Y}$ denote the union of their possible values, including Boolean, categorical, and integer-valued domains. We represent the node-level annotation set as
\begin{equation}
A=\{(v,f,y)\mid v\in V,\ f\in\mathcal{F},\ y\in\mathcal{Y}\},
\end{equation}
where $v$ is a CIT node, $f$ is a node-level field such as \lstinline|Sketch_type|, \lstinline|Is_connected|, etc., and $y$ is the expected value.

Relation-level constraints describe expected relationships between entities or subtrees. Let $\mathcal{G}$ denote the set of relation types. These include \lstinline|Contact_type|, \lstinline|Intersection|, \lstinline|Has_coplanar_faces|, and \lstinline|Has_shared_axis|. We represent the relation annotation set as
\begin{equation}
R=\{(v_i,v_j,g,y)\mid v_i,v_j\in V,\ g\in\mathcal{G},\ y\in\mathcal{Y}\},
\end{equation}
where $v_i$ and $v_j$ are participating CIT nodes, $g$ is a relation type, and $y$ is the expected relation value, such as \lstinline|face| for a face-contact relation or \lstinline|True| for a shared-axis relation.

The CIT is intentionally not a full CAD program. It omits low-level numeric details when they are underspecified in natural language, but retains the construction intent that should be preserved by the generated code. This makes the representation suitable both for generation and for validation.

\subsection{CIT Constraint Annotation}
\label{subsec:cit-constraints}

CIT-CAD annotates each inferred CIT with a compact constraint vocabulary that can be checked against generated CAD programs. The vocabulary is divided into entity-level constraints and relation-level constraints. Entity-level constraints describe local properties of individual construction entities, while relation-level constraints describe spatial or semantic relations among entities. This annotation is part of the method pipeline: the expected constraints used for verification are derived from the CIT inferred from the input specification, not directly from the reference CAD program.

\cref{tab:cit-constraints} lists the constraint vocabulary used in the current implementation. Entity-level constraints include \lstinline|Sketch_type|, \lstinline|Is_connected|, \lstinline|Is_open_profile|, and \lstinline|Boolean_role|. Relation-level constraints include \lstinline|Intersection|, \lstinline|Contact_type|, \lstinline|Has_coplanar_faces|, and \lstinline|Has_shared_axis|. These constraints are selected because they are common in multi-entity CAD programs and can be deterministically extracted from generated code through static analysis and runtime geometric reasoning.

For each entity node, \lstinline|Sketch_type| identifies the primitive or composite sketch class, \lstinline|Is_connected| checks whether the sketch consists of a single connected profile, \lstinline|Is_open_profile| records whether the profile is represented as an open drawing chain before closure, and \lstinline|Boolean_role| indicates whether the entity contributes additively to the final model or serves as a subtractive cutting feature. For pairs or groups of entities, relation-level constraints check whether entities overlap, touch and how they touch, share coplanar faces, or share circular/cylindrical axes. These constraints capture construction semantics that are difficult to evaluate using final-shape overlap alone.

\subsection{CIT Inference from Natural Language}
\label{subsec:cit-inference}

Given a natural-language description $D$, CIT-CAD first prompts an LLM to infer a CIT in JSON format. The prompt specifies the node schema, node-level constraints, relation-level constraints, and naming conventions. It instructs the model to use stable snake-case identifiers, create one node for every physical solid or cutting tool, and use null values for uncertain constraints rather than hallucinating precise geometry.

Formally, the CIT inference stage estimates:
\begin{equation}
T = f_{\theta}^{\text{tree}}(D),
\end{equation}
where $f_{\theta}^{\text{tree}}$ is an LLM used as a structured information extractor. The output is parsed as JSON and normalized to ensure that it contains a root node, expected variables, node constraints, relation constraints, and child nodes. 

Although the CIT is inferred by an LLM, downstream verification does not trust the LLM output blindly. Instead, the CIT is treated as the explicit design intent extracted from the natural-language specification. CIT-CAD then uses deterministic procedures to derive constraints from this tree and to check generated code against them.

\subsection{Constraint-Guided CAD Code Generation}
\label{subsec:code-generation}

After inferring the CIT, CIT-CAD generates CAD programs using the natural-language description and the CIT jointly:
\begin{equation}
P_0 = f_{\theta}^{\text{code}}(D,T).
\end{equation}
The code-generation prompt requires the generated program to import \lstinline|cadquery as cq|, create one variable for every entity node in the CIT, use the node's expected variable name, construct each entity through a Workplane sketch followed by \lstinline|.extrude(...)|, compose additive and subtractive entities according to their \lstinline|Boolean_role|, and assign the final object to a variable named \lstinline|solid|.

Conditioning code generation on the CIT has two benefits. First, it gives the LLM an explicit construction scaffold, reducing ambiguity in entity naming, part decomposition, and boolean composition. Second, it makes the generated code more analyzable: because expected entities are named in the CIT, the validator can compare extracted code constraints against the intended nodes and relations.

\subsection{Deterministic Constraint Extraction}
\label{subsec:constraint-extraction}

CIT-CAD derives two comparable short-constraint sets: the expected constraint set $C_T$ from the CIT $T$, and the executed constraint set $C_P$ from the generated program $P$. The expected set is obtained by normalizing the CIT annotations:
\begin{equation}
C_T=\mathrm{Normalize}(A,R)=(E_T,Q_T).
\end{equation}
Here, $A$ and $R$ are the raw node-level and relation-level intent annotations stored on the CIT, $E_T$ is the expected entity-constraint map, and $Q_T$ is the expected relation-constraint set. The conversion maps each CIT node $v$ to its expected code variable $n$, removes unspecified or null-valued annotations, and canonicalizes relation participants so that pairwise relation comparison is order-invariant.

\noindent\textbf{Expected constraints from the CIT.}
The CIT-to-constraint conversion is deterministic. For each node-level annotation $(v,f,y)\in A$, if node $v$ maps to expected variable $n$, CIT-CAD creates the entity constraint $E_T(n,f)=y$. For each relation annotation $(v_i,v_j,g,y)\in R$, if $v_i$ and $v_j$ map to expected variables $n_i$ and $n_j$, CIT-CAD creates the canonical relation constraint $(g,\operatorname{sort}(n_i,n_j),y)\in Q_T$. The converter also records a node map that links each constraint back to its CIT node path, enabling later localization.

\noindent\textbf{Executed constraints from generated code.}
To extract $C_P$, CIT-CAD analyzes the generated CAD program using deterministic program and geometry analysis. Static AST analysis identifies variables assigned from \lstinline|.extrude(...)|, recovers sketch operations, classifies \lstinline|Sketch_type|, and computes \lstinline|Boolean_role| from \lstinline|union|, \lstinline|cut|, and \lstinline|intersect| expressions. Runtime geometric reasoning executes the CAD program and inspects generated geometry to determine relation-level constraints such as \lstinline|Contact_type|, \lstinline|Intersection|, \lstinline|Has_coplanar_faces|, and \lstinline|Has_shared_axis|. The extracted result is normalized into the same short-constraint format as $C_T$.

Formally, for a constraint source $X\in\{T,P\}$, where $X=T$ denotes the expected source induced by the CIT and $X=P$ denotes the executed source extracted from the generated program, let $\mathcal{N}_X$ be the set of entity variables observed in that source. CIT-CAD represents entity constraints from source $X$ as a partial map
\begin{equation}
E_X: \mathcal{N}_X \times \mathcal{F} \rightharpoonup \mathcal{Y},
\end{equation}
where $E_X(n,f)=y$ means that entity variable $n\in\mathcal{N}_X$ has field $f\in\mathcal{F}$ with value $y\in\mathcal{Y}$. Relation constraints from source $X$ are represented as canonical triples
\begin{equation}
Q_X=\{(g,\operatorname{sort}(\mathbf{u}),y)\mid g\in\mathcal{G},\mathbf{u}\in\mathcal{N}_X^2,y\in\mathcal{Y}\}.
\end{equation}
Here, $\mathbf{u}$ is the tuple of participating entity variables, and $\operatorname{sort}(\mathbf{u})$ canonicalizes their order before comparison. The normalized short-constraint set for source $X$ is therefore
\begin{equation}
C_X=(E_X,Q_X).
\end{equation}
This gives $C_T=(E_T,Q_T)$ for the CIT-derived expected constraints and $C_P=(E_P,Q_P)$ for the generated-program constraints.

This design separates uncertain LLM-based interpretation from deterministic validation. The LLM may infer or generate an imperfect result, but the verification step uses a fixed analyzer and geometry engine to decide whether the generated program satisfies the CIT-derived requirements.

\subsection{Structural and Semantic Verification}
\label{subsec:verification}

Given expected constraints $C_T=(E_T,Q_T)$ and executed constraints $C_P=(E_P,Q_P)$, CIT-CAD verifies the generated program using exact equality over the normalized short-constraint representation. For an expected entity variable $n$ and field $f$, the entity-field satisfaction predicate is
\begin{equation}
\operatorname{sat}_{\mathrm{ent}}(n,f)=
\begin{cases}
1, & n\in\mathcal{N}_P \ \land\ E_P(n,f)=E_T(n,f),\\
0, & \text{otherwise}.
\end{cases}
\label{eq:entity-satisfaction}
\end{equation}
If an expected entity variable does not appear in $P$, all constraints associated with that entity are marked as violated and the feedback records a missing-entity violation.

For relation constraints, CIT-CAD first canonicalizes the participating entity names and then checks set membership:
\begin{equation}
\operatorname{sat}_{\mathrm{rel}}(g,\mathbf{u},y)
=
\mathbb{I}\left[(g,\operatorname{sort}(\mathbf{u}),y)\in Q_P\right],
\label{eq:relation-satisfaction}
\end{equation}
for every expected relation triple $(g,\operatorname{sort}(\mathbf{u}),y)\in Q_T$. The generated program satisfies the CIT if and only if all expected entity-field constraints and all expected relation constraints are satisfied:
\begin{equation}
\begin{aligned}
\operatorname{Valid}(P,T)
=& \bigwedge_{(n,f)\in\operatorname{dom}(E_T)} \operatorname{sat}_{\mathrm{ent}}(n,f) \\
&\land \bigwedge_{(g,\mathbf{u},y)\in Q_T} \operatorname{sat}_{\mathrm{rel}}(g,\mathbf{u},y).
\end{aligned}
\label{eq:validity}
\end{equation}

The validator also constructs explicit satisfied and violated constraint sets. Let $\mathcal{C}_T$ denote the finite set of individual expected constraints contained in $C_T$:
\begin{equation}
\mathcal{C}_T=
\{(n,f,E_T(n,f))\mid(n,f)\in\operatorname{dom}(E_T)\}\cup Q_T.
\end{equation}
\noindent CIT-CAD partitions $\mathcal{C}_T$ into a satisfied set and a violated set:
\begin{equation}
S(P,T)=\{c\in\mathcal{C}_T\mid\operatorname{sat}(c,P)=1\},
V(P,T)=\mathcal{C}_T\setminus S(P,T).
\label{eq:constraint-partition}
\end{equation}
Here, \(S(P,T)\) is the set of expected constraints satisfied by program \(P\), \(V(P,T)\) is the set of expected constraints violated by \(P\), and \(\operatorname{sat}(c,P)\) denotes \cref{eq:entity-satisfaction} when \(c\) is an entity-field constraint and \cref{eq:relation-satisfaction} when \(c\) is a relation constraint. This partition is exactly the information later consumed by the repair loop.

The validator produces three outputs. First, it reports \(\operatorname{Valid}(P,T)\). Second, it records satisfied and violated constraint identifiers.
Third, it localizes every violated constraint to the corresponding CIT node, relation, or subtree using the node map produced during CIT-to-constraint conversion.

This localization is important because it turns validation failures into actionable repair feedback. Instead of reporting only that the final geometry is inaccurate, CIT-CAD can indicate that a specific node is missing, a specific sketch field is wrong, or a specific relation between two CIT nodes is not satisfied.

\subsection{Localized Reflexion Repair}
\label{subsec:repair}

CIT-CAD uses localized validation feedback to iteratively repair the generated program. After each validation round, the feedback contains violated entities, violated relations, their corresponding CIT paths, and concrete repair guidance. The LLM repair prompt receives the natural-language description, the CIT, the current code, and the validation feedback, and returns a revised CAD program.

However, CIT-CAD does not assume that the LLM repair model is reliable. A repair may fix one violated constraint while accidentally breaking a previously satisfied one. To prevent such regression, CIT-CAD wraps the LLM repair model in a monotonic constraint-preserving acceptance loop, shown in \cref{alg:repair}.

Let $K$ be the maximum number of repair attempts, and let $P_t$ be the accepted program at repair iteration $t$. The deterministic validator partitions $\mathcal{C}_T$ into satisfied and violated constraints for $P_t$:
\begin{equation}
S_t = S(P_t,T), \quad V_t=V(P_t,T).
\end{equation}
Here, $S_t$ and $V_t$ are shorthand for the satisfied and violated sets at iteration $t$. The implementation maintains a locked set $L_t$ of constraints that have been satisfied by any accepted program so far:
\begin{equation}
L_t = \bigcup_{i=0}^{t} S_i.
\label{eq:locked-set}
\end{equation}
The validator also converts $V_t$ into localized repair feedback $F_t$, which records the violated constraints and their corresponding CIT nodes, relations, or subtrees. After the LLM proposes a candidate repair $P'_t$, CIT-CAD accepts the candidate only if both of the following conditions hold:
\begin{equation}
L_t \subseteq S(P'_t,T).
\label{eq:locked-preservation}
\end{equation}
\begin{equation}
|V(P'_t,T)| < |V_t|.
\label{eq:strict-progress}
\end{equation}
If either condition fails, the candidate is rejected and the previous program $P_t$ is retained. Thus, the LLM may propose arbitrary changes, but only validation-approved changes are committed.

\begin{algorithm}[t]
\caption{Monotonic CIT-guided repair}
\label{alg:repair}
\KwIn{Description $D$, CIT $T$, initial program $P_0$, max repair attempts $K$}
\KwOut{Verified or best accepted CAD program $P$}
$C_T \leftarrow \mathrm{ExtractConstraints}(T)$\,, $P \leftarrow P_0$\,, $L \leftarrow \emptyset$\;
\For{$t \leftarrow 0$ \KwTo $K-1$}{
    $(S_t,V_t) \leftarrow \mathrm{Validate}(P,C_T)$\;
    \If{$V_t=\emptyset$}{
        \Return $P$\;
    }
    $L \leftarrow L \cup S_t$\;
    $F_t \leftarrow \mathrm{Localize}(V_t,T)$\;
    $P' \leftarrow \mathrm{LLMRepair}(D,T,P,F_t,L)$\;
    $(S',V') \leftarrow \mathrm{Validate}(P',C_T)$\;
    \If{$L \subseteq S'$ \textbf{and} $|V'|<|V_t|$}{
        $P \leftarrow P'$\;
    }
    \If{$|V'|\ge |V_t|$}{
        \Return $P$\;
    }
}
\Return $P$\;
\end{algorithm}

\begin{proposition}[Regression-free accepted repairs]
\label{prop:regression-free}
Assume the validator is deterministic and CIT-CAD accepts a candidate repair only when \cref{eq:locked-preservation} holds. Then no accepted repair violates a constraint that was satisfied in the previous accepted program.
\end{proposition}

\begin{proof}
At iteration $t$, every constraint satisfied by the current accepted program belongs to the locked set because $S_t\subseteq L_t$. A candidate $P'_t$ is accepted only if $L_t\subseteq S(P'_t,T)$. Therefore, every constraint satisfied by $P_t$ remains satisfied after accepting $P'_t$. If any constraint in $L_t$ becomes violated, the candidate fails the acceptance rule and is rejected, so the accepted program remains $P_t$. Thus, accepted repairs are regression-free with respect to the CIT-derived constraint set.
\end{proof}

The second acceptance condition, \cref{eq:strict-progress}, ensures strict progress in the number of remaining violations for every accepted repair. Since the CIT-derived constraint set is finite, the number of accepted repairs is bounded by the number of initially violated constraints, although the system may still terminate earlier due to a fixed iteration budget or repeated rejected candidates.

\section{Dataset and Metrics}
\label{sec:dataset}

We construct our evaluation set from Text2CAD~\cite{khan2024text2cad} to study construction-aware CAD code generation. Each evaluated sample contains a natural-language description, a held-out reference CAD program, and generated CAD programs from the evaluated methods. This design allows us to evaluate generated programs from both geometry-level and construction-level perspectives without using the reference program as input to generation or CIT inference.

\subsection{Data Source and Sample Selection}
\label{subsec:dataset-source}

We use Text2CAD~\cite{khan2024text2cad} as our benchmark dataset. 
Text2CAD contains about 178K text-to-CAD samples. 
Following its preprocessing protocol, we remove samples with missing natural-language descriptions, missing CAD programs, or duplicate descriptions, resulting in 151K valid text-to-CAD pairs.

Since CIT-CAD targets construction-aware generation, we focus on multi-entity CAD programs. 
We parse each reference CAD program and count the number of explicit extruded construction entities. 
Samples with only one entity are excluded because they usually involve limited construction hierarchy, Boolean composition, or inter-entity relations.
The final evaluation subset contains 26,783 multi-entity samples. 
Their entity-count distribution is 17,193 samples with two entities, 5,966 with three entities, 1,883 with four entities, 731 with five entities, 671 with six entities, and 339 with seven or more entities. 
This long-tail distribution allows us to evaluate both common simple multi-entity designs and more complex cases that require preserving richer construction structure.

For each selected sample, the natural-language description is used as the input specification. 
The reference CAD program is used only for geometry-level evaluation and entity-count grouping, and is never provided to the code generator or the CIT inference stage.

\subsection{Evaluation Metrics}
\label{subsec:dataset-metrics}

We evaluate generated CAD programs using four complementary metrics: valid syntax rate, geometry-level success rate, Intersection over Union, and constraint satisfaction rate.

\paragraph{\textbf{Valid Syntax Rate}}
Valid syntax rate (VSR) measures the percentage of generated programs that can be executed successfully by the CadQuery runtime and converted into a valid solid for evaluation. Programs with syntax errors, unsupported CadQuery calls, runtime exceptions, or missing final solids are counted as invalid. VSR reflects whether a method can generate executable CAD code.

\paragraph{\textbf{Geometry-level metrics}}
We use Intersection over Union (IoU) to measure final-shape similarity between the generated model and the reference model. Both models are normalized before comparison, and IoU is computed over their occupied 3D regions:
\begin{equation}
\mathrm{IoU}(G, R) =
\frac{\mathrm{Vol}(G \cap R)}
{\mathrm{Vol}(G \cup R)},
\end{equation}
where \(G\) and \(R\) denote the generated and reference solids, respectively. We report mean IoU and median IoU over samples whose generated programs execute successfully and whose IoU values fall in the valid range \([0,1]\). We also report Success Rate (IoU@0.95), defined as the percentage of all evaluated samples whose generated geometry reaches an IoU greater than 0.95 with the reference geometry.

\paragraph{\textbf{Constraint Satisfaction Rate}}
We introduce Constraint Satisfaction Rate (CSR) to measure construction-level correctness based on CIT-derived structural and semantic constraints. For each sample \(i\), let \(C_i\) be the deterministic constraint set derived from the CIT inferred from the natural-language description, and let \(\hat{C}_i\) be the constraint set extracted from the generated program. CSR is defined as
\begin{equation}
\mathrm{CSR}_i =
\frac{|\{c \in C_i \mid c \text{ is satisfied by } \hat{C}_i\}|}
{|C_i|}.
\end{equation}
The final CSR is averaged over the evaluation set. Unlike IoU, CSR evaluates whether the generated program preserves the intended construction process, including entity decomposition, sketch properties, Boolean roles, and inter-entity relations. Therefore, CSR complements geometry-level metrics rather than replacing them.


In our experiments, each natural-language description is used to infer a CIT and its corresponding construction constraints. 
These CIT-derived constraints are used for CSR computation and for checking whether the generated CAD program preserves the intended construction process. 
The reference CAD program is used only for geometry-level evaluation, such as building the reference solid for IoU computation. 
This separation ensures that constraint-level evaluation is based on the input design intent, while the reference program remains a held-out target for final-geometry comparison.
\section{Experiments}
\label{sec:experiments}

This section evaluates whether CIT-CAD improves natural-language CAD code generation across different backbone LLMs and different levels of construction complexity. We introduce the experimental settings and organize the evaluation around the following research questions.

\subsection{Experimental Setup}
\label{subsec:exp-setup}

\noindent\textbf{Evaluated methods.}
For each backbone LLM, we compare two generation settings. The \textit{Vanilla} directly prompts the LLM to generate CadQuery code from the natural-language description. \textit{CIT-CAD} first infers a Constraint Intent Tree, generates CadQuery code conditioned on the tree, and then applies constraint-guided validation and repair. This comparison isolates the effect of introducing explicit construction intent and deterministic feedback while keeping the underlying LLM family fixed.

\noindent\textbf{Backbone LLMs.}
We evaluate CIT-CAD with two open-source LLMs, Qwen3-A3B-Instruct~\cite{yang2025qwen3} and DeepSeek-Coder-V2-Lite-Instruct~\cite{guo2024deepseek}, and one closed-source LLM, GPT-5.4-mini~\cite{openai2025gpt5mini}. The open-source models are served through vLLM, while GPT-5.4-mini is accessed through an OpenAI-compatible API. Each backbone is used for both the direct Vanilla and the CIT-CAD pipeline.

\begin{table*}[ht]
\centering
\caption{RQ1 results by LLM backbone and reference entity count. Signed percentages after CIT-CAD results indicate relative change over the corresponding Vanilla setting, with green for improvement and red for degradation.}
\label{tab:rq1-entity-llm}
\footnotesize
\setlength{\tabcolsep}{4.5pt}
\renewcommand{\arraystretch}{1.08}
\resizebox{\textwidth}{!}{%
\begin{tabular}{lllllllll}
\toprule
\rowcolor{gray!12}
\textbf{Method} & \textbf{Metric} & \multicolumn{7}{c}{\textbf{Number of reference entities}} \\
\cmidrule(lr){3-9}
\rowcolor{gray!12}
& & \textbf{Entity=2} & \textbf{Entity=3} & \textbf{Entity=4} & \textbf{Entity=5} & \textbf{Entity=6} & \textbf{Entity$\geq$7} & \textbf{All} \\
\rowcolor{gray!6}
& \textit{\#Number of Samples} & 17193 & 5966 & 1883 & 731 & 671 & 339 & 26783 \\
\midrule
\rowcolor{highlight2}
\multicolumn{9}{l}{\textit{DeepSeek-Coder-V2-Lite-Instruct}} \\
\multirow{5}{*}{\textit{Vanilla}} & Valid Syntax Rate & 77.0\% & 73.5\% & 68.1\% & 62.1\% & 49.6\% & 48.7\% & 74.1\% \\
& Success Rate~(IoU@0.95) & 4.2\% & 1.5\% & 0.7\% & \textbf{0.4\%} & 0.6\% & \textbf{0.3\%} & 3.1\% \\
& Mean IoU & \textbf{32.6\%} & 25.7\% & 22.6\% & 16.6\% & 8.5\% & \textbf{10.3\%} & 29.0\% \\
& Median IoU & 24.6\% & 17.6\% & 14.9\% & 10.1\% & \textbf{2.0\%} & \textbf{1.3\%} & 21.1\% \\
& Constraint Satisfaction Rate & 14.8\% & 15.0\% & 15.5\% & 16.4\% & 13.8\% & 9.6\% & 14.8\% \\
\cmidrule(lr){1-9}
\multirow{5}{*}{\textit{CIT-CAD}} & Valid Syntax Rate & \textbf{85.3\%}{\tiny\textcolor{green!50!black}{(+10.8\%)}} & \textbf{82.5\%}{\tiny\textcolor{green!50!black}{(+12.2\%)}} & \textbf{78.3\%}{\tiny\textcolor{green!50!black}{(+15.0\%)}} & \textbf{76.1\%}{\tiny\textcolor{green!50!black}{(+22.5\%)}} & \textbf{68.7\%}{\tiny\textcolor{green!50!black}{(+38.5\%)}} & \textbf{67.5\%}{\tiny\textcolor{green!50!black}{(+38.6\%)}} & \textbf{83.3\%}{\tiny\textcolor{green!50!black}{(+12.3\%)}} \\
& Success Rate~(IoU@0.95) & \textbf{4.5\%}{\tiny\textcolor{green!50!black}{(+7.1\%)}} & \textbf{1.6\%}{\tiny\textcolor{green!50!black}{(+6.7\%)}} & \textbf{0.8\%}{\tiny\textcolor{green!50!black}{(+14.3\%)}} & \textbf{0.4\%} {\tiny (0.0\%)} & \textbf{0.8\%}{\tiny\textcolor{green!50!black}{(+33.3\%)}} & \textbf{0.3\%} {\tiny (0.0\%)} & \textbf{3.3\%}{\tiny\textcolor{green!50!black}{(+7.3\%)}} \\
& Mean IoU & 32.4\%{\tiny\textcolor{red}{(-0.6\%)}} & \textbf{25.9\%}{\tiny\textcolor{green!50!black}{(+0.8\%)}} & \textbf{25.0\%}{\tiny\textcolor{green!50!black}{(+10.6\%)}} & \textbf{19.9\%}{\tiny\textcolor{green!50!black}{(+19.9\%)}} & \textbf{8.8\%}{\tiny\textcolor{green!50!black}{(+3.5\%)}} & 8.4\%{\tiny\textcolor{red}{(-18.4\%)}} & \textbf{29.2\%}{\tiny\textcolor{green!50!black}{(+0.5\%)}} \\
& Median IoU & \textbf{26.7\%}{\tiny\textcolor{green!50!black}{(+8.5\%)}} & \textbf{17.8\%}{\tiny\textcolor{green!50!black}{(+1.1\%)}} & \textbf{15.0\%}{\tiny\textcolor{green!50!black}{(+0.7\%)}} & \textbf{13.8\%}{\tiny\textcolor{green!50!black}{(+36.6\%)}} & 1.6\%{\tiny\textcolor{red}{(-20.0\%)}} & 0.8\%{\tiny\textcolor{red}{(-38.5\%)}} & \textbf{22.6\%}{\tiny\textcolor{green!50!black}{(+7.0\%)}} \\
& Constraint Satisfaction Rate & \textbf{27.6\%}{\tiny\textcolor{green!50!black}{(+86.5\%)}} & \textbf{28.1\%}{\tiny\textcolor{green!50!black}{(+87.3\%)}} & \textbf{29.7\%}{\tiny\textcolor{green!50!black}{(+91.6\%)}} & \textbf{30.9\%}{\tiny\textcolor{green!50!black}{(+88.4\%)}} & \textbf{35.4\%}{\tiny\textcolor{green!50!black}{(+156.5\%)}} & \textbf{36.1\%}{\tiny\textcolor{green!50!black}{(+276.0\%)}} & \textbf{28.3\%}{\tiny\textcolor{green!50!black}{(+90.3\%)}} \\
\midrule
\rowcolor{highlight2}
\multicolumn{9}{l}{\textit{Qwen3-A3B-Instruct}} \\
\multirow{5}{*}{\textit{Vanilla}} & Valid Syntax Rate & 80.3\% & 77.4\% & 74.2\% & 65.3\% & 64.4\% & 46.3\% & 78.0\% \\
& Success Rate~(IoU@0.95) & 7.2\% & 3.0\% & 1.5\% & 1.0\% & 0.9\% & 0.3\% & 5.5\% \\
& Mean IoU & \textbf{36.0\%} & 28.8\% & 25.0\% & 19.5\% & 10.0\% & \textbf{9.2\%} & 32.6\% \\
& Median IoU & 27.0\% & 20.2\% & \textbf{17.7\%} & 14.3\% & 1.9\% & \textbf{2.5\%} & 23.3\% \\
& Constraint Satisfaction Rate & 17.5\% & 18.0\% & 15.0\% & 14.3\% & 5.4\% & 6.6\% & 16.7\% \\
\cmidrule(lr){1-9}
\multirow{5}{*}{\textit{CIT-CAD}} & Valid Syntax Rate & \textbf{87.1\%}{\tiny\textcolor{green!50!black}{(+8.5\%)}} & \textbf{85.1\%}{\tiny\textcolor{green!50!black}{(+10.0\%)}} & \textbf{80.7\%}{\tiny\textcolor{green!50!black}{(+8.7\%)}} & \textbf{74.3\%}{\tiny\textcolor{green!50!black}{(+13.8\%)}} & \textbf{73.3\%}{\tiny\textcolor{green!50!black}{(+13.9\%)}} & \textbf{54.6\%}{\tiny\textcolor{green!50!black}{(+17.8\%)}} & \textbf{85.1\%}{\tiny\textcolor{green!50!black}{(+9.1\%)}} \\
& Success Rate~(IoU@0.95) & \textbf{7.6\%}{\tiny\textcolor{green!50!black}{(+5.7\%)}} & \textbf{3.6\%}{\tiny\textcolor{green!50!black}{(+21.3\%)}} & \textbf{1.7\%}{\tiny\textcolor{green!50!black}{(+10.3\%)}} & \textbf{1.1\%}{\tiny\textcolor{green!50!black}{(+14.3\%)}} & \textbf{1.0\%}{\tiny\textcolor{green!50!black}{(+16.7\%)}} & \textbf{0.6\%}{\tiny\textcolor{green!50!black}{(+100.0\%)}} & \textbf{5.9\%}{\tiny\textcolor{green!50!black}{(+7.9\%)}} \\
& Mean IoU & \textbf{36.0\%} {\tiny (0.0\%)} & \textbf{29.7\%}{\tiny\textcolor{green!50!black}{(+3.1\%)}} & \textbf{25.3\%}{\tiny\textcolor{green!50!black}{(+1.2\%)}} & \textbf{20.4\%}{\tiny\textcolor{green!50!black}{(+4.2\%)}} & \textbf{10.6\%}{\tiny\textcolor{green!50!black}{(+6.2\%)}} & \textbf{9.2\%} {\tiny (0.0\%)} & \textbf{32.8\%}{\tiny\textcolor{green!50!black}{(+0.6\%)}} \\
& Median IoU & \textbf{27.0\%}{\tiny\textcolor{green!50!black}{(+0.1\%)}} & \textbf{21.1\%}{\tiny\textcolor{green!50!black}{(+4.0\%)}} & 17.2\%{\tiny\textcolor{red}{(-2.5\%)}} & \textbf{14.3\%}{\tiny\textcolor{green!50!black}{(+0.1\%)}} & \textbf{2.1\%}{\tiny\textcolor{green!50!black}{(+11.4\%)}} & 2.0\%{\tiny\textcolor{red}{(-20.0\%)}} & \textbf{23.5\%}{\tiny\textcolor{green!50!black}{(+0.9\%)}} \\
& Constraint Satisfaction Rate & \textbf{30.1\%}{\tiny\textcolor{green!50!black}{(+71.8\%)}} & \textbf{32.1\%}{\tiny\textcolor{green!50!black}{(+78.5\%)}} & \textbf{33.9\%}{\tiny\textcolor{green!50!black}{(+126.4\%)}} & \textbf{34.2\%}{\tiny\textcolor{green!50!black}{(+139.9\%)}} & \textbf{36.1\%}{\tiny\textcolor{green!50!black}{(+572.0\%)}} & \textbf{43.3\%}{\tiny\textcolor{green!50!black}{(+560.7\%)}} & \textbf{31.8\%}{\tiny\textcolor{green!50!black}{(+90.1\%)}} \\
\midrule
\rowcolor{highlight2}
\multicolumn{9}{l}{\textit{GPT-5.4-mini}} \\
\multirow{5}{*}{\textit{Vanilla}} & Valid Syntax Rate & 88.5\% & 89.2\% & 86.7\% & 78.3\% & 69.8\% & 61.5\% & 87.2\% \\
& Success Rate~(IoU@0.95) & 5.6\% & 3.6\% & 1.7\% & 0.8\% & 0.5\% & 0.3\% & 4.6\% \\
& Mean IoU & 33.9\% & 28.9\% & 27.3\% & 27.6\% & 9.3\% & 4.8\% & 31.2\% \\
& Median IoU & 24.9\% & 20.1\% & 22.2\% & 17.3\% & 2.2\% & 0.8\% & 22.6\% \\
& Constraint Satisfaction Rate & 18.0\% & 14.5\% & 16.8\% & 10.0\% & 11.5\% & 11.4\% & 16.7\% \\
\cmidrule(lr){1-9}
\multirow{5}{*}{\textit{CIT-CAD}} & Valid Syntax Rate & \textbf{91.9\%}{\tiny\textcolor{green!50!black}{(+3.8\%)}} & \textbf{92.0\%}{\tiny\textcolor{green!50!black}{(+3.1\%)}} & \textbf{90.9\%}{\tiny\textcolor{green!50!black}{(+4.9\%)}} & \textbf{89.4\%}{\tiny\textcolor{green!50!black}{(+14.2\%)}} & \textbf{72.7\%}{\tiny\textcolor{green!50!black}{(+4.2\%)}} & \textbf{81.0\%}{\tiny\textcolor{green!50!black}{(+31.5\%)}} & \textbf{90.4\%}{\tiny\textcolor{green!50!black}{(+3.7\%)}} \\
& Success Rate~(IoU@0.95) & \textbf{6.8\%}{\tiny\textcolor{green!50!black}{(+22.0\%)}} & \textbf{3.9\%}{\tiny\textcolor{green!50!black}{(+8.3\%)}} & \textbf{2.7\%}{\tiny\textcolor{green!50!black}{(+58.8\%)}} & \textbf{1.1\%}{\tiny\textcolor{green!50!black}{(+37.5\%)}} & \textbf{1.3\%}{\tiny\textcolor{green!50!black}{(+160.0\%)}} & \textbf{0.9\%}{\tiny\textcolor{green!50!black}{(+200.0\%)}} & \textbf{5.5\%}{\tiny\textcolor{green!50!black}{(+20.7\%)}} \\
& Mean IoU & \textbf{34.9\%}{\tiny\textcolor{green!50!black}{(+2.9\%)}} & \textbf{29.3\%}{\tiny\textcolor{green!50!black}{(+1.4\%)}} & \textbf{29.6\%}{\tiny\textcolor{green!50!black}{(+8.4\%)}} & \textbf{29.1\%}{\tiny\textcolor{green!50!black}{(+5.4\%)}} & \textbf{15.5\%}{\tiny\textcolor{green!50!black}{(+66.7\%)}} & \textbf{12.0\%}{\tiny\textcolor{green!50!black}{(+150.0\%)}} & \textbf{32.3\%}{\tiny\textcolor{green!50!black}{(+3.8\%)}} \\
& Median IoU & \textbf{30.7\%}{\tiny\textcolor{green!50!black}{(+23.3\%)}} & \textbf{24.6\%}{\tiny\textcolor{green!50!black}{(+22.4\%)}} & \textbf{27.6\%}{\tiny\textcolor{green!50!black}{(+24.3\%)}} & \textbf{24.5\%}{\tiny\textcolor{green!50!black}{(+41.6\%)}} & \textbf{3.9\%}{\tiny\textcolor{green!50!black}{(+77.3\%)}} & \textbf{1.3\%}{\tiny\textcolor{green!50!black}{(+62.5\%)}} & \textbf{27.9\%}{\tiny\textcolor{green!50!black}{(+23.7\%)}} \\
& Constraint Satisfaction Rate & \textbf{37.4\%}{\tiny\textcolor{green!50!black}{(+107.8\%)}} & \textbf{37.5\%}{\tiny\textcolor{green!50!black}{(+158.6\%)}} & \textbf{37.0\%}{\tiny\textcolor{green!50!black}{(+120.2\%)}} & \textbf{36.3\%}{\tiny\textcolor{green!50!black}{(+263.0\%)}} & \textbf{37.2\%}{\tiny\textcolor{green!50!black}{(+223.5\%)}} & \textbf{50.0\%}{\tiny\textcolor{green!50!black}{(+338.6\%)}} & \textbf{37.7\%}{\tiny\textcolor{green!50!black}{(+125.7\%)}} \\
\bottomrule
\end{tabular}%
}
\end{table*}

\subsection{Quantitative Analysis}
\begin{researchquestion}
How do LLM backbone and CAD entity complexity affect CIT-CAD performance?
\label{rq:llm-entity}
\end{researchquestion}

\cref{tab:rq1-entity-llm} reports the results grouped by reference entity count. Across the open-source LLMs, CIT-CAD consistently improves execution success over direct generation. The improvement is especially clear as entity count increases. 


The evaluated backbones show different behaviors. DeepSeek-Coder benefits strongly from the CIT-CAD wrapper in terms of VSR and CSR. Its CSR improves by 86.5\% on two-entity samples and by 276.0\% on samples with seven or more entities, indicating that the repaired programs satisfy substantially more extracted CIT constraints. Qwen3 also obtains consistent VSR gains and stronger Success Rate in most entity groups, suggesting that explicit construction intent and deterministic feedback are useful not only for weaker code models but also for stronger CAD code generators. GPT-5.4-mini shows a smaller but positive aggregate effect after excluding unavailable generated files: CIT-CAD improves overall VSR by 3.7\%, Success Rate by 19.6\%, mean IoU by 3.5\%, median IoU by 23.5\%, and CSR by 125.7\%.

\findingbox{
CIT-CAD consistently improves the quality of CAD models across different backbones. The gain becomes larger as construction complexity increases for the open-source backbones, showing that explicit Constraint Intent Trees are particularly helpful for multi-entity CAD programs. The CSR results show that CIT-CAD substantially improves constraint-level structural validity.
}

\begin{researchquestion}
How about the correlation between geometric-level and constraint-level correctness?
\label{rq:iou-csr-correlation}
\end{researchquestion}

RQ1 shows that IoU and CSR do not always move in the same direction. To quantify this gap, we pair each valid generated model's normalized IoU with its corresponding CSR and compute sample-level Pearson and Spearman correlations. \cref{fig:rq2-csr-iou-correlation} visualizes the Qwen3 results at the sample level. The correlation between CSR and IoU is close to zero, with Pearson $r=0.022$ and Spearman $\rho=0.023$. This indicates that final-shape overlap does not reliably predict whether the generated CAD program preserves construction intent. We also observe high-IoU/low-CSR cases in which the final solid is geometrically similar to the reference but the program omits expected entities, assigns incorrect Boolean roles, or uses a different sketch decomposition. These cases motivate CSR as a complementary metric to IoU.

\findingbox{
Geometric-level similarity does not strongly reflect constraint-level correctness. IoU-only evaluation can therefore overestimate text-to-CAD generation quality, especially when a generated program produces a similar final shape through a structurally different construction process.
}

\begin{figure}
    \centering
    \includegraphics[width=\linewidth]{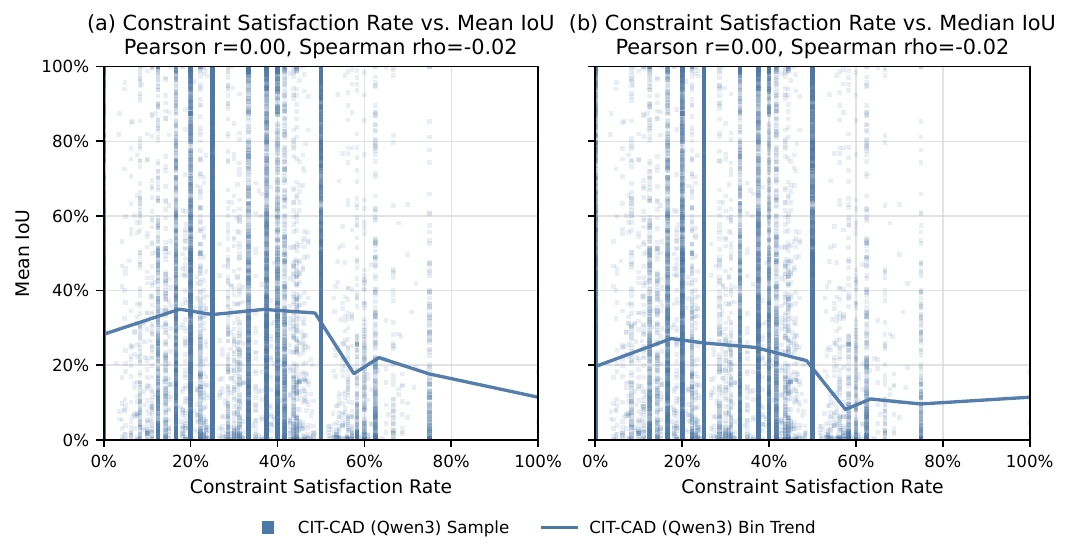}
    \caption{Correlation between geometric-level similarity and constraint-level correctness on Qwen3. Each point is one generated sample. The weak Pearson and Spearman correlations show that mean/median IoU and CSR capture different aspects of CAD program quality.}
    \label{fig:rq2-csr-iou-correlation}
\end{figure}

\begin{researchquestion}
How much does constraint-guided repair improve generated CAD programs?
\label{rq:repair-effectiveness}
\end{researchquestion}

CIT-CAD does not stop after the first code generation attempt. It validates the generated program against the extracted Constraint Intent Tree, converts violated constraints into localized feedback, and asks the same backbone LLM to repair the program. A repaired candidate is accepted only when it preserves locked constraints and reduces the number of remaining violations. If a candidate fails to introduce newly satisfied constraints or reduces no remaining violation, the repair loop terminates early for that sample. For stopped samples, the last accepted CSR is carried forward in later iterations, so each iteration is evaluated over the same fixed sample cohort.

\cref{fig:rq3-csr-distribution} shows the CSR distribution of Qwen3 across repair iterations. The average CSR increases monotonically from 18.2\% before repair to 21.6\%, 25.0\%, 28.4\%, and finally 28.9\% after iterative repair. This corresponds to a 10.7\% absolute improvement, or a 58.8\% relative improvement over the initial CSR. The distribution also shifts upward across iterations, indicating that localized constraint feedback improves not only a small subset of samples but the overall constraint-satisfaction profile.

The improvement is most pronounced in the early repair rounds. The first three repair iterations each bring clear gains, while the improvement from Iteration 3 to Iteration 4 becomes smaller. This suggests that constraint-guided repair quickly fixes many actionable violations, such as missing or mismatched structural constraints, but later iterations face harder cases where remaining errors are less easily resolved by textual feedback alone. The diminishing gain also supports the use of early stopping: once a repair candidate no longer reduces the violation set, continuing to modify the program is unlikely to provide substantial benefit and may risk unnecessary changes.

\begin{figure}[t]
\centering
\includegraphics[width=0.9\linewidth]{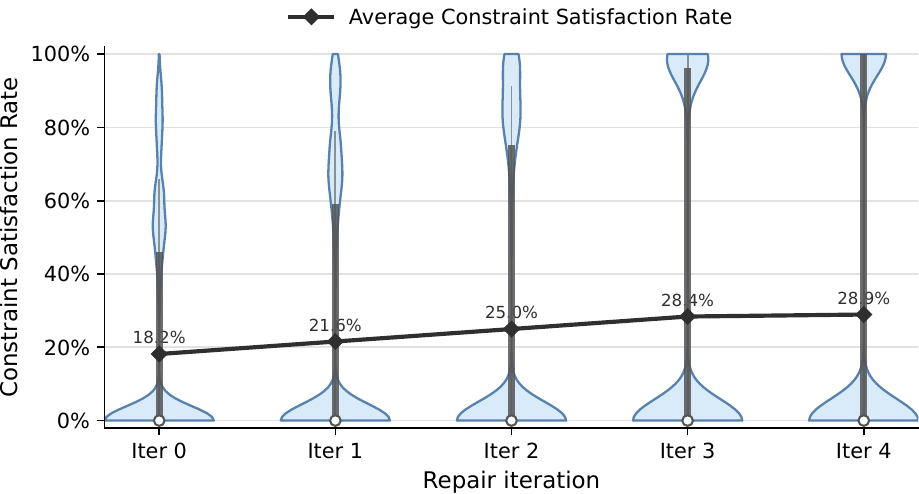}
\caption{Constraint Satisfaction Rate Distribution Across Repair Iterations of Qwen3. The violin plots show the CSR distribution over all evaluated samples at each repair iteration, while the line reports the average CSR.}
\label{fig:rq3-csr-distribution}
\end{figure}

\findingbox{
Constraint-guided repair turns CIT-CAD from a one-shot generator into an iterative verifier-generator. On Qwen3, the repair loop improves the average CSR from 18.2\% to 28.9\%, demonstrating that deterministic validation and localized feedback can progressively increase constraint satisfaction. The gains are strongest in the first few iterations and gradually saturate.
}

\begin{figure*}[t]
    \centering
    \includegraphics[width=\linewidth]{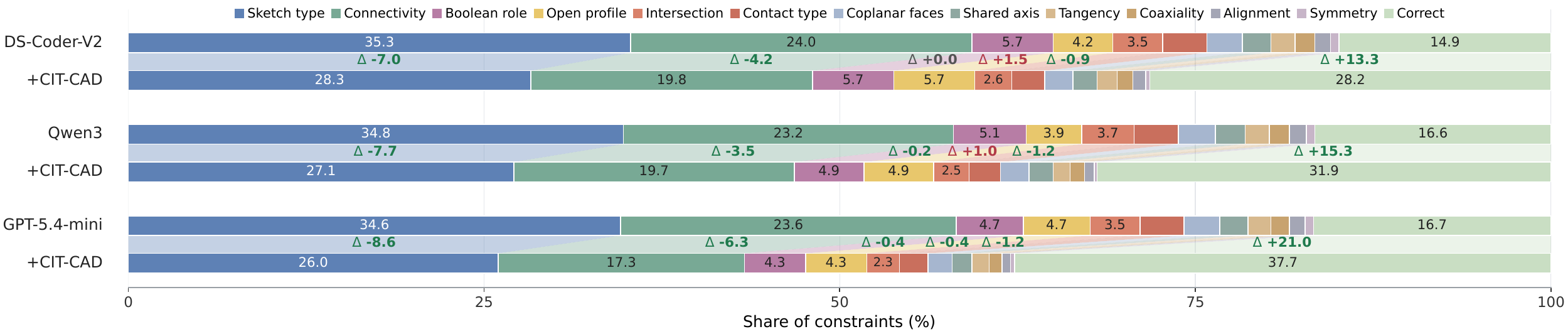}
    \caption{Normalized distribution of violated constraints. Each horizontal bar decomposes the number of violated constraints per 100 extracted constraints by constraint type. Lower total length indicates fewer construction-intent violations.}
    \label{fig:rq4-failure-types}
\end{figure*}

\begin{researchquestion}
What is the distribution of violated construction constraints?
\label{rq:failure-analysis}
\end{researchquestion}


We further analyze violated construction constraints to reveal program-level failure modes beyond final-shape IoU. 
We group violations by CIT field, such as sketch type, Boolean role, connectivity, contact, coplanarity, alignment, symmetry, and other geometric relations. 
Since one sample may violate multiple constraints, we report normalized counts as violations per 100 extracted constraints.

\cref{fig:rq4-failure-types} shows two main patterns. First, CIT-CAD reduces the total normalized violation count across all three backbones. The reduction is visible for DeepSeek-Coder-V2, Qwen3, and GPT-5.4-mini, indicating that explicit CIT scaffolding improves construction-intent preservation beyond a single model family. Second, the dominant remaining violations are sketch-level constraints, especially \lstinline|Sketch_type| and \lstinline|Is_connected|. This suggests that CIT-CAD is effective at imposing a global entity scaffold, but fine-grained local sketch reconstruction remains the hardest part of the task.

Relation-level violations, such as contact type, coplanarity, shared axis, tangency, coaxiality, alignment, and symmetry, account for smaller portions of the normalized violation counts. 
This does not imply that such relations are less important. 
Many relation checks depend on first generating the correct entities and sketch profiles. 
Thus, the main bottleneck is improving sketch-type grounding and local profile construction while preserving the CIT-level entity and relation scaffold.

\findingbox{
CIT-CAD reduces normalized constraint violations across all evaluated backbones, but it still fails most often on fine-grained sketch-level constraints. The remaining errors are dominated by sketch type and profile connectivity, while relation-level violations form a smaller but still meaningful portion of the error distribution. These results point to future improvements in fine-grained sketch extraction, parameter grounding, and relation-aware repair.
}

\begin{figure}[t]
    \centering
    \includegraphics[width=\linewidth]{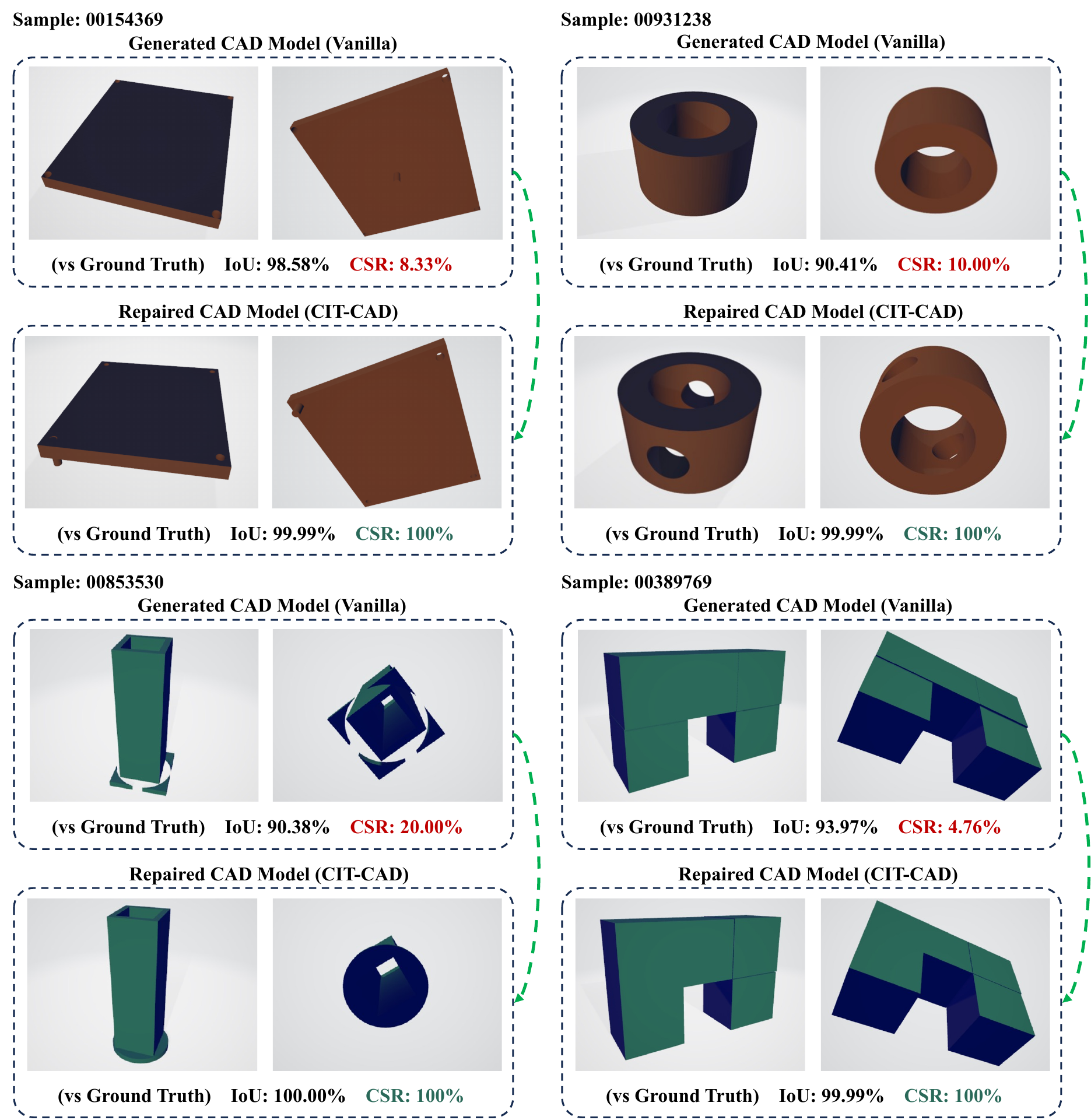}
    \caption{Qualitative case study of CIT-CAD on representative multi-entity CAD samples.}
    \label{fig:case-study}
\end{figure}

\subsection{Qualitative Analysis~(Case Study)}
\begin{researchquestion}
What benefits does CIT-CAD bring on typical multi-entity CAD cases?
\label{rq:case-study}
\end{researchquestion}

\cref{fig:case-study} provides qualitative evidence for the quantitative findings.
Although Vanilla generation can sometimes produce a visually plausible shape with high IoU, it may still fail to preserve the intended construction process.
For example, in Sample~00931238, the Vanilla model generates the main cylindrical body and the central hole, but misses the side through-hole.
This indicates that the model captures the coarse geometry while failing to execute an important subtractive construction step, leading to a low CSR of 10.00\% despite an IoU of 90.41\%.

CIT-CAD makes such errors explicit through the Constraint Intent Tree.
The omitted side hole is detected as a violated constraint related to a missing or incorrect subtractive entity.
This localized feedback provides direct repair guidance: instead of regenerating the whole model blindly, CIT-CAD can focus on correcting the violated construction role and inserting the required cut operation.
After repair, the model recovers the intended side-hole structure, improving IoU to 99.99\% and CSR to 100\%.

\findingbox{
The case study shows that CIT-CAD improves CAD generation by exposing construction-level errors that are not fully reflected by IoU.
Vanilla may produce a globally similar shape while omitting entities or confusing additive/subtractive roles.
By converting these errors into violated constraints, CIT-CAD provides localized and interpretable repair signals, leading to programs that better match the intended construction scaffold.
}

\section{Conclusion}
\label{sec:conclusion}

This paper presented CIT-CAD, a Constraint Intent Tree-based framework for natural-language CAD code generation and verification. The central idea is to make construction intent explicit before generating code: CIT-CAD infers a tree-structured representation of expected CAD entities, operations, sketch attributes, Boolean roles, and inter-entity relations, then uses this representation to guide CadQuery generation and deterministic validation. By comparing CIT-derived expected constraints with constraints extracted from generated programs, CIT-CAD provides a program-level view of CAD correctness that complements final-shape IoU.
Our experiments show that CIT-CAD improves the quality of generated CAD programs from both geometric-level and constraint-level across different LLM backbones. 

Future work will extend the constraint detector vocabulary to cover more fine-grained construction relations. 
We will also improve parameter-aware CIT inference and develop repair models that can better target local sketch and relation errors.

\clearpage

\bibliographystyle{ieeetr}
\bibliography{cad}

\clearpage

\twocolumn[
  \begin{@twocolumnfalse}
    \begin{center}
      \fontsize{18}{20}\selectfont\bfseries
          Supplementary Material \\ Constraint Intent Tree-based CAD code Generation and Validation
      \vspace{1em}
    \end{center}
  \end{@twocolumnfalse}
]

\setcounter{page}{1}

This appendix provides supplementary details for reproducing and interpreting CIT-CAD. It first reports the prompt templates used for CIT extraction, CIT-guided CAD code generation, and constraint-guided repair. It then formalizes the short-constraint validation procedure, describes the dataset preprocessing and parameter-augmentation pipeline, summarizes the experimental configuration, and discusses the main threats to validity.

\section{Prompt Templates}
\label{app:prompt-templates}

This appendix provides the inference-time prompt templates used by CIT-CAD. We report the prompt content at the level needed for reproduction while omitting private deployment details such as API keys, local endpoints, and machine-specific paths. CIT-CAD uses two prompt stages: an initial generation stage that first extracts the Constraint Intent Tree and then generates CadQuery code conditioned on the tree, and a repair stage that revises the generated code using localized constraint-validation feedback.

\subsection{Constraint Intent Tree Extraction}
\label{app:generation-prompts}

The initial generation stage contains two calls. The first call extracts a Constraint Intent Tree from the natural-language description. The second call generates executable CadQuery code from the description and the extracted tree.

\begin{tcolorbox}[
    breakable,
    colback=white,
    colframe=gray!40!black,
    title=CIT Extraction Prompt,
    fonttitle=\bfseries
]
\small
\raggedright
\textbf{System Prompt:} \\
\texttt{You extract a CAD structure tree from a natural-language CAD description.}

\textbf{Instruction:} \\
\texttt{Output requirements:} \\
\texttt{1. Return only valid JSON. Do not include Markdown fences.} \\
\texttt{2. Include the sample id, the original description, and a root node.} \\
\texttt{3. Use stable snake-case ids and expected\_variable names.} \\
\texttt{4. Create one child node for every physical solid or cutting tool.} \\
\texttt{5. For each entity, extract sketch\_type, is\_connected, is\_open\_profile, and boolean\_role when supported by the description.} \\
\texttt{6. For relations, extract intersection, contact\_type, has\_coplanar\_faces, and has\_shared\_axis when they are explicitly implied or strongly entailed.} \\
\texttt{7. Use null values for uncertain entity constraints rather than guessed geometry.} \\
\texttt{8. Use boolean\_role=true for additive solids and boolean\_role=false for subtractive cutting tools.} \\
\texttt{9. Use sketch\_type values from rect, circle, with\_arc, only\_line, and other.}

\textbf{User Input:} \\
\texttt{Sample id:} \{sample\_id\} \\
\texttt{Natural-language description:} \{description\}
\end{tcolorbox}

\subsection{CIT-Guided CAD Code Generation}

\begin{tcolorbox}[
    breakable,
    colback=white,
    colframe=gray!40!black,
    title=CIT-Guided CAD Code Generation Prompt,
    fonttitle=\bfseries
]
\small
\raggedright
\textbf{System Prompt:} \\
\texttt{You generate executable CadQuery Python code from a natural-language CAD description and a Constraint Intent Tree.}

\textbf{Instruction:} \\
\texttt{Output requirements:} \\
\texttt{1. Return only Python code. Do not include Markdown fences.} \\
\texttt{2. Import cadquery as cq.} \\
\texttt{3. Create one variable for every non-group node in the Constraint Intent Tree.} \\
\texttt{4. Use each node's expected\_variable exactly as the variable name.} \\
\texttt{5. Each entity variable should be produced by a Workplane sketch followed by .extrude(...).} \\
\texttt{6. Compose entities using .union(...), .cut(...), or .intersect(...) according to the intended construction.} \\
\texttt{7. Assign the final object to a variable named solid.} \\
\texttt{8. Prefer simple explicit dimensions when the description is underspecified.} \\
\texttt{9. Preserve the entity constraints and relation constraints in the Constraint Intent Tree as much as possible.} \\
\texttt{10. Do not call .close() after .rect(...), .circle(...), .ellipse(...), or helpers that already create closed wires.} \\
\texttt{11. Use .close() only after explicit open chains such as .moveTo(...).lineTo(...).lineTo(...).}

\textbf{User Input:} \\
\texttt{Natural-language description:} \{description\} \\
\texttt{Constraint Intent Tree:} \{tree\_json\}
\end{tcolorbox}

\subsection{Constraint-Guided Repair}
\label{app:repair-prompt}

The repair stage receives the natural-language description, the inferred CIT, the current generated code, and deterministic constraint-validation feedback. The feedback identifies violated entity or relation constraints, their localized CIT paths, and constraints that have already been satisfied and should be preserved.

\begin{tcolorbox}[
    breakable,
    colback=white,
    colframe=gray!40!black,
    title=Constraint-Guided Repair Prompt,
    fonttitle=\bfseries
]
\small
\raggedright
\textbf{System Prompt:} \\
\texttt{You repair CadQuery Python code using localized Constraint Intent Tree validation feedback.}

\textbf{Instruction:} \\
\texttt{Output requirements:} \\
\texttt{1. Return only the revised Python code. Do not include Markdown fences.} \\
\texttt{2. Keep the same intended object and preserve variable names from the Constraint Intent Tree.} \\
\texttt{3. Focus on the violated entity constraints, relation constraints, or subtrees listed in the feedback.} \\
\texttt{4. The final object must be assigned to a variable named solid.} \\
\texttt{5. Preserve constraints that were already satisfied in previous validation rounds.} \\
\texttt{6. Prefer localized edits to the violated entities, relations, or subtrees.} \\
\texttt{7. A candidate repair must preserve satisfied constraints and reduce the number of violated constraints.} \\
\texttt{8. Do not call .close() after .rect(...), .circle(...), .ellipse(...), or helpers that already create closed wires.} \\
\texttt{9. Use .close() only after explicit open chains such as .moveTo(...).lineTo(...).lineTo(...).} \\
\texttt{10. If feedback mentions ``Cannot convert object type ... Wire ... to vector'', remove the extra .close() before .extrude(...).}

\textbf{User Input:} \\
\texttt{Natural-language description:} \{description\} \\
\texttt{Constraint Intent Tree:} \{tree\_json\} \\
\texttt{Current code:} \{code\} \\
\texttt{Constraint-validation feedback:} \{feedback\_json\}
\end{tcolorbox}

\section{Constraint Extraction and Verification Details}
\label{app:constraint-verification}

CIT-CAD converts both expected and generated CAD artifacts into the same short-constraint representation before verification. The expected representation is produced from the CIT. For each entity node, the converter reads the serializable implementation keys \lstinline|sketch_type|, \lstinline|is_connected|, \lstinline|is_open_profile|, and \lstinline|boolean_role| when they are present and non-null. For each relation attached to a node, the converter keeps relation types in \{\lstinline|intersection|, \lstinline|contact_type|, \lstinline|has_coplanar_faces|, \lstinline|has_shared_axis|\}, sorts the two participating entity names, and stores a triple containing entities and value. A node map is also stored so each expected constraint can later be traced back to a CIT node path.

The generated program is analyzed by a deterministic CadQuery analyzer. Static AST analysis identifies variables assigned from \lstinline|.extrude(...)| expressions, recovers sketch construction calls, classifies local sketch fields, and infers Boolean roles from composition expressions such as \lstinline|union|, \lstinline|cut|, and \lstinline|intersect|. Runtime geometric analysis executes the generated program and inspects the resulting solids to extract relation constraints. The analyzer then strips metadata and emits the same short representation used for expected constraints.

\begin{algorithm}[t]
\caption{Short-constraint validation}
\label{alg:short-validation}
\KwIn{Expected constraints $C_T$, actual constraints $C_P$}
\KwOut{Validation result, satisfied IDs, violated IDs}
$S \leftarrow \emptyset$, $V \leftarrow \emptyset$\;
\ForEach{expected entity $n$ in $C_T$}{
    \If{$n$ is missing from $C_P$}{
        $V \leftarrow V \cup \{\mathrm{entity}:n:\mathrm{exists}\}$\;
        \textbf{continue}\;
    }
    \ForEach{expected field $f$ of $n$}{
        \If{$C_P[n][f] = C_T[n][f]$}{
            $S \leftarrow S \cup \{\mathrm{entity}:n:f\}$\;
        }
        \Else{
            $V \leftarrow V \cup \{\mathrm{entity}:n:f\}$\;
        }
    }
}
\ForEach{relation type $g$}{
    $A_g \leftarrow$ canonicalized actual relation triples of type $g$\;
    \ForEach{expected triple $(g,\mathbf{u},y)$}{
        \If{$(g,\operatorname{sort}(\mathbf{u}),y)\in A_g$}{
            $S \leftarrow S \cup \{\mathrm{relation}:g:\mathbf{u}:y\}$\;
        }
        \Else{
            $V \leftarrow V \cup \{\mathrm{relation}:g:\mathbf{u}:y\}$\;
        }
    }
}
\Return{$V=\emptyset$, $S$, $V$}\;
\end{algorithm}

Algorithm~\ref{alg:short-validation} mirrors the implemented validation procedure. Entity constraints are checked by exact field equality. Relation constraints are checked by membership in a canonicalized relation set. This intentionally avoids fuzzy matching: if the CIT expects a face-contact relation and the generated code produces no contact or a different contact type, the relation is counted as violated even if the final geometry has non-zero IoU with the reference model.

\section{Dataset Preprocessing}
\label{app:dataset-preprocessing}

The evaluation set is derived from Text2CAD~\cite{khan2024text2cad}. We first join the natural-language CSV and CadQuery CSV by normalized sample ID, remove samples with missing descriptions or missing CAD programs, and remove duplicate natural-language descriptions. This yields 151K valid text-to-CAD pairs, as reported in the main paper.

CIT-CAD is designed for construction-aware multi-entity generation. We therefore filter the valid pairs using an implementation-level entity parser. The parser counts explicit extruded construction entities in the reference CAD program and keeps only samples with at least two entities. To ensure that the natural-language description exposes the intended construction complexity, the preprocessing script also estimates the entity count mentioned in the description using numeric and ordinal expressions and keeps samples whose natural-language entity count matches the code entity count. Samples whose reference programs cannot be parsed by the entity parser are excluded from the final selected set. After preprocessing, the final evaluation subset contains 26,783 samples.

The natural-language inputs used by the final experiments include parameter augmentation. The augmentation script prompts an LLM to extract compact parameter summaries from the held-out reference CadQuery code, including visible dimensions, radii, heights, translations, rotations, entity counts, and Boolean composition. The summary is appended to the natural-language description before generation. This step addresses a limitation of the original descriptions, which often describe the object qualitatively but omit numeric parameters needed for executable CadQuery code. The reference code is not provided to the code-generation model at evaluation time; only the augmented natural-language description is used as input.

\section{Experimental Configuration}
\label{app:experimental-config}

All evaluated methods consume the same augmented natural-language descriptions. For each backbone LLM, the Vanilla baseline directly prompts the model to generate a standalone CAD program and does not use CITs, constraints, or repair. CIT-CAD runs the complete pipeline: CIT extraction, CIT-conditioned code generation, deterministic validation, localized repair, and monotonic candidate acceptance.

The open-source backbones are served through vLLM-compatible OpenAI-style chat endpoints. The Qwen setting uses Qwen3-30B-A3B-Instruct, and the DeepSeek setting uses DeepSeek-Coder-V2-Lite-Instruct in the reported experiments. GPT-5.4-mini is accessed through an OpenAI-compatible closed-source API. All LLM calls use chat-completion format with temperature 1.0 when the endpoint supports temperature; if an endpoint rejects the temperature argument, the implementation automatically retries without it. The maximum retry count is three. Dataset-level generation and repair are parallelized across samples, with 16 workers used by default and 32 workers used in some large closed-source runs.

The maximum number of repair iterations is controlled by \lstinline|--max-reflexion-iters|. The default implementation uses two repair iterations, while some extended GPT runs use five iterations. Existing generated samples are skipped unless explicit regeneration is requested. This skip policy is used to support long-running experiments and resume interrupted generations without overwriting completed samples. For geometry-level evaluation, generated and reference solids are normalized to the same voxel grid before IoU computation, and invalid or missing generated programs are counted as failures for validity and success-rate metrics.

\section{Threats to Validity}
\label{app:threats}

We discuss the main factors that may affect the interpretation and generalizability of our results. These threats mainly arise from the quality of the inferred CITs, the coverage of the deterministic constraint detectors, the construction-aware dataset filtering strategy, and the execution-based nature of CAD code evaluation.

\paragraph{CIT extraction quality.}
CIT-CAD treats the inferred CIT as the explicit design intent extracted from the natural-language input. If the extraction model misses an entity or infers an incorrect relation, the verifier will faithfully check the wrong intent. We reduce this risk by using a fixed output format, requiring null values for uncertain constraints, and evaluating across multiple LLM backbones, but the quality of CIT extraction remains an important source of variance.

\paragraph{Constraint detector coverage.}
CSR measures satisfaction of the constraints currently supported by the deterministic analyzer. The implementation covers entity-level sketch fields, Boolean role, intersection, contact type, coplanar faces, and shared axes. It does not yet cover all possible CAD design intents, such as precise symmetry, tolerances, manufacturing constraints, or arbitrary parametric dependencies. CSR should therefore be interpreted as construction-intent correctness under the implemented vocabulary, not as a complete proof of CAD equivalence.

\paragraph{Dataset and parameter augmentation.}
The final dataset is filtered to multi-entity samples whose natural-language entity count matches the code-derived entity count. This selection is appropriate for evaluating construction-aware generation, but it may underrepresent single-solid designs or descriptions with implicit entities. Parameter augmentation improves executability by making numeric details available in the natural-language input, but it also means the evaluated setting is closer to parameter-aware specification following than to generation from purely qualitative prompts.

\paragraph{Execution-based evaluation.}
Valid syntax rate, IoU, and CSR all depend on successful CadQuery execution and deterministic geometry analysis. Runtime failures are treated as invalid outputs and can affect both geometry-level and constraint-level metrics. This is intentional for CAD code generation, where executable code is the target artifact, but it may penalize models for small API-level errors even when the intended shape is partially recoverable.

\end{document}